\documentclass[letterpaper]{article}
\usepackage{aaai}
\usepackage{times}
\usepackage{helvet}
\usepackage{courier}

\usepackage{amsmath}
\usepackage{amssymb}
\usepackage{bm}
\usepackage{graphicx}
\usepackage[toc,page]{appendix}
\usepackage{booktabs}

\DeclareMathOperator*{\argmax}{argmax}
\DeclareMathOperator*{\diag}{diag}
\DeclareMathOperator{\tr}{tr}

\usepackage{subfigure}

\usepackage[amsmath,thmmarks]{ntheorem}
\usepackage{theorem}
\newtheorem{theorem}{Theorem}
\newtheorem{lemma}{Lemma}

\newtheorem{definition}{Definition}

\theoremheaderfont{\sc}\theorembodyfont{\upshape}
\theoremstyle{nonumberplain}
\theoremseparator{}
\theoremsymbol{\rule{1ex}{1ex}}
\newtheorem{proof}{Proof}

\begin{document}
 
%
\title{Block-Wise MAP Inference for Determinantal Point Processes\\with Application to Change-Point Detection}
\author{Jinye Zhang \and Zhijian Ou\\
Speech Processing and Machine Intelligence Laboratory\\
Tsinghua University, Beijing, China, 100085\\
}
\nocopyright
\frenchspacing
\maketitle
\begin{abstract}
\begin{quote}
Existing MAP inference algorithms for determinantal point processes (DPPs) need to calculate determinants or conduct eigenvalue decomposition generally at the scale of the full kernel, which presents a great challenge for real-world applications. In this paper, we introduce a class of DPPs, called BwDPPs, that are characterized by an almost block diagonal kernel matrix and thus can allow efficient block-wise MAP inference. Furthermore, BwDPPs are successfully applied to address the difficulty of selecting change-points in the problem of change-point detection (CPD), which results in a new BwDPP-based CPD method, named BwDppCpd. In BwDppCpd, a preliminary set of change-point candidates is first created based on existing well-studied metrics. Then, these change-point candidates are treated as DPP items, and DPP-based subset selection is conducted to give the final estimate of the change-points that favours both quality and diversity. The effectiveness of BwDppCpd is demonstrated through extensive experiments on five real-world datasets. 
\end{quote}
\end{abstract}

\section{Introduction}
The determinantal point processes (DPPs) are elegant probabilistic models for subset selection problems where both quality and diversity are considered. Formally, given a set of items $\mathcal{Y}=\{1,\cdots,N\}$, a DPP defines a probability measure $\mathcal{P}$ on $2^{\mathcal{Y}}$, the set of all subsets of $\mathcal{Y}$. For every subset $Y \subseteq \mathcal{Y}$ we have 
\begin{equation}
\mathcal{P}_{\mathbf{L}}(Y)\propto\det({\mathbf{L}_Y}),
\end{equation}
where the L-ensemble kernel $\mathbf{L}$ is an $N$ by $N$ positive semi-definite matrix. By writing $\mathbf{L}=\mathbf{B}^T \mathbf{B}$ as a Gram matrix, $\det({\mathbf{L}_Y})$ could be viewed as the squared volume spanned by the column vectors $\mathbf{B}_i$ for $i\in Y$. By defining $\mathbf{B}_i=q_i \bm{\phi}_i$, a popular decomposition of the kernel is given as
\begin{equation}
L_{ij}=q_i \bm{\phi}_i^T \bm{\phi}_j q_j,
\end{equation}
where $q_i\in \mathbb{R}^+$ measures the quality (magnitude) of item $i$ in $\mathcal{Y}$, and $\bm{\phi}_i\in \mathbb{R}^k$, $\Vert \bm{\phi}_i\Vert=1$ can be viewed as the angle vector of diversity features so that $\bm{\phi}_i^T \bm{\phi}_j$ measures the similarity between items $i$ and $j$. It can be shown that the probability of including $i$ and $j$ increases with the quality of $i$ and $j$ and diversity between $i$ and $j$. As a result, a DPP assigns high probability to subsets that are both of good quality and diverse \cite{kulesza2012determinantal}.

\begin{figure}
\centering
\subfigure[~]{\includegraphics[width=5cm]{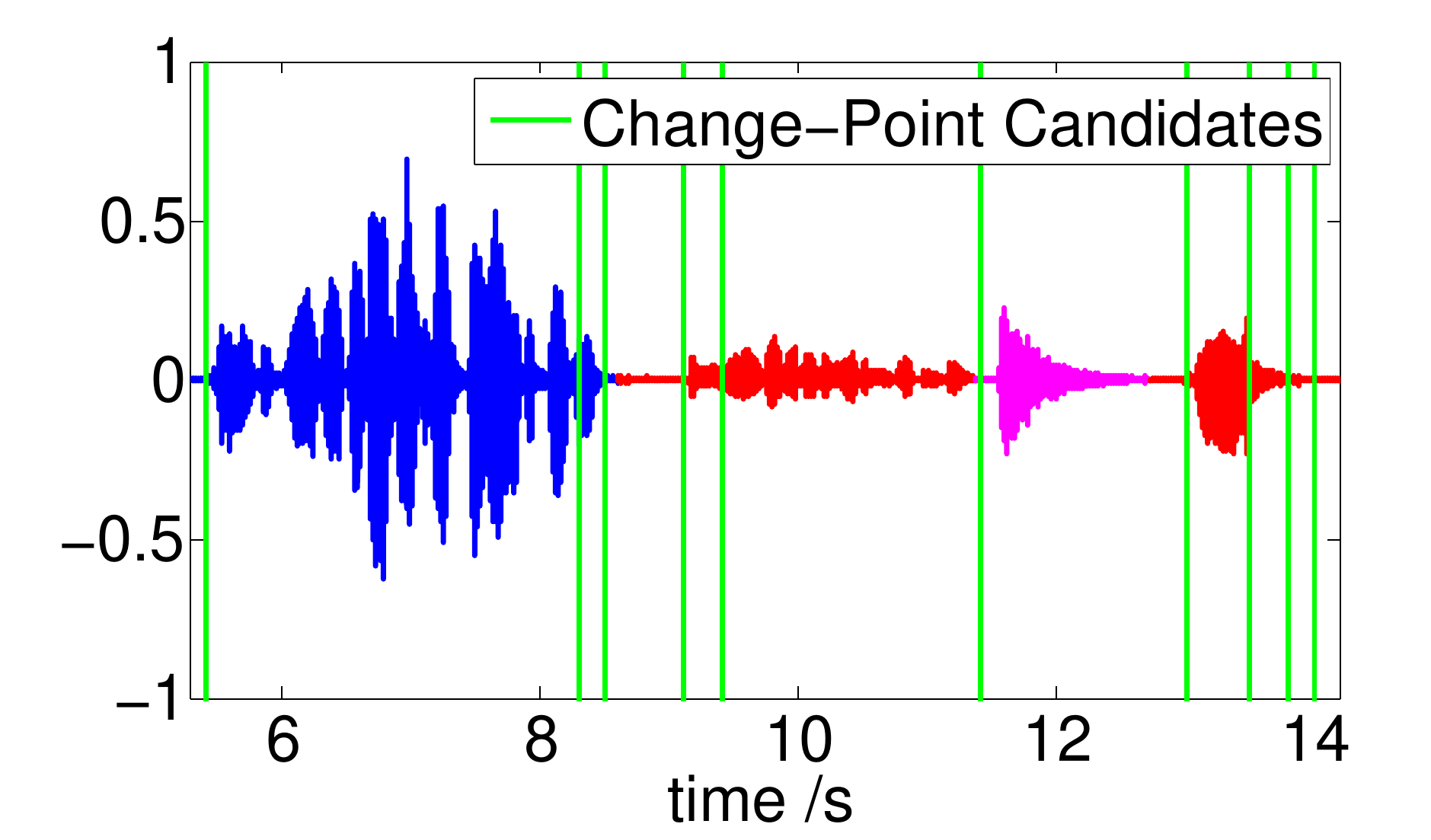}} 
\subfigure[~]{\includegraphics[width=3cm]{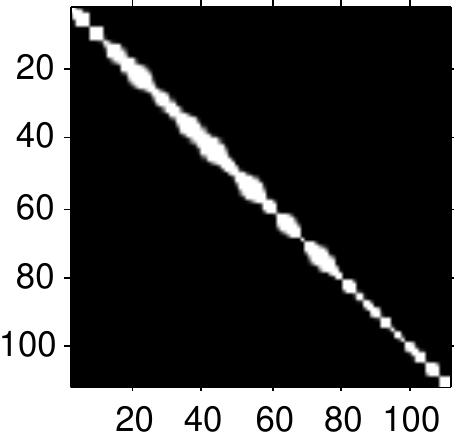}}
\caption{(a) A 10-sec part of a 2-min speech recording, shown with change-point candidates. Segments of different speakers or noises are plotted in different colors. (b) BwDPP kernel constructed for the whole 2-min recording, with the 112 change-point candidates as BwDPP items. The white denotes non-zero entries while the black indicates zero.}
\label{fig: TelData}
\end{figure}

For DPPs, the \emph{maximum a posteriori} (MAP) problem $\argmax_{Y \subseteq \mathcal{Y}}  \det(\mathbf{L}_Y)$, aiming at finding the subset with highest probability, has attracted much attention due to its broad range for potential applications. Noting that this is an NP-hard problem \cite{ko1995exact}, a number of approximate inference methods have been purposed, including the greedy methods for optimizing the submodular function $\log \det(L_Y)$ \cite{buchbinder2012tight,nemhauser1978analysis}, optimization via continuous relaxation \cite{gillenwater2012near}, and minimum Bayes risk decoding that minimizes the application-specific loss function \cite{kulesza2012determinantal}. 

These existing methods need to calculate determinants or conduct eigenvalue decomposition. Both computations are taken at the scale of the kernel size $N$ and with the cost of around $\mathcal{O}(N^3)$ time that become intolerably high when $N$ become large, e.g. thousands. Nevertheless, we find that for a class of DPPs where the kernel is almost block diagonal (Fig. \ref{fig: TelData} (b)), the MAP inference with the whole kernel could be replaced by a series of sub-inferences with its sub-kernels. Since the sizes of the sub-kernels become smaller, the overall computational cost can be significantly reduced. Such DPPs are often defined over a line where items are only similar to their neighbourhoods on the line and significantly different from those far away. Since the MAP inference for such DPPs is conducted in a block-wise manner, we refer to them as BwDPPs (block-wise DPPs) in the rest of the paper.

The above observation is mainly motivated by the problem of change-point detection (CPD) that aims at detecting abrupt changes in time-series data \cite{gustafsson2000adaptive}. In CPD, the period of time between two consecutive change-points, often referred to as a segment or a state, is with homogeneous properties of interest (e.g. the same speaker in a speech \cite{chen1998speaker} or the same behaviour in human activity data \cite{liu2013change}). After choosing a number of change-point candidates without much difficulty, we can treat these change-point candidates as DPP items, and select a subset from them to be our final estimate of the change-points. Each change-point candidate has its own quality of being a change-point. Moreover, the true locations of change-points along the timeline tend to be diverse, since states (e.g. speakers in Fig. \ref{fig: TelData} (a)) would not change rapidly. Therefore, it is preferred to conduct change-point selection that incorporates both quality and diversity. DPP-based subset selection clearly suits this purpose well. Meanwhile, the corresponding kernel will then become almost block diagonal (e.g. Fig. \ref{fig: TelData} (b)), as neighbouring items are less diversified, and items far apart more diversified, In this case, the DPP becomes BwDPP.

The problem of CPD have been actively studied for decades, where various CPD methods could be broadly classified into Bayesian or frequentist approach. In Bayesian approach, the CPD problem is reduced to estimating the posterior distribution of the change-point locations given the time-series data \cite{green1995reversible}. Other posteriors to be estimated include the 0/1 indicator sequence \cite{lavielle2001application}, and the ``run length" \cite{adams2007bayesian}. Although many improvements were made, e.g. using advanced Monte Carlo method, the efficiency for estimating these posteriors is still a big challenge for real-world tasks.


In frequentist approach, the core idea is hypothesis testing and the general strategy is to first define a metric (test statistic) by considering the observations over past and present windows. As both windows move forward, change-points are selected when the metric value exceeds a threshold. Some widely-used metrics include the cumulative sum \cite{basseville1993detection}, the generalized likelihood-ratio \cite{gustafsson1996marginalized}, the Bayesian information criterion (BIC) \cite{chen1998speaker}, the Kullback Leibler divergence \cite{delacourt2000distbic}, and more recently, subspace-based metrics \cite{ide2007change,kawahara2007change}, kernel-based metrics \cite{desobry2005online}, and density-ratio \cite{kanamori2010theoretical,kawahara2012sequential}. While various metrics have been explored, how to choose thresholds and perform change-point selection, which is also a determining factor for detection performance, is relatively less studied. Heuristic-based rules or procedures are dominant and not well-performed, e.g. selecting local peaks above a threshold \cite{kawahara2007change}, discarding the lower one if two peaks are close \cite{liu2013change}, or requiring the metric differences between change-points and their neighbouring valleys above a threshold \cite{delacourt2000distbic}. 


In this paper, we propose to apply DPP to address the difficulty of selecting change-points. Based on existing well-studied metrics, we can create a preliminary set of change-point candidates without much difficulty. Then, we treat these change-point candidates as DPP items, and conduct DPP-based subset selection to obtain the final estimate of the change-points that favours both quality and diversity.

The contribution of this paper is two-fold. First, we introduce a class of DPP, called BwDPPs, that are characterized by an almost block diagonal kernel matrix and thus can allow efficient block-wise MAP inference. Second, BwDPPs are successfully applied to address the difficult problem of selecting change-points, which results in a new BwDPP-based CPD method, named BwDppCpd. 

The rest of the paper is organized as follows. After describing brief preliminaries, we introduce BwDPPs and give our theoretical result on the BwDPP-MAP method. Next, we introduce BwDppCpd and present evaluation experiment results on a number of real-world datasets. Finally, we conclude the paper with a discussion on potential future directions.
\section{Preliminaries \label{Sec: Pre}} 
Throughout the paper, we are interested in MAP inference for BwDPPs, a particular class of DPP where the L-ensemble kernel $\mathbf{L}$ is almost block diagonal\footnote{Such matrices could also be defined as a particular class of block tridiagonal matrices, where the off-diagonal sub-matrices $\mathbf{A}_i$ only have a few non-zeros entries at the bottom left.}, namely 
\begin{equation}\label{def: abd}
\mathbf{L}\triangleq\left[ \begin{array}{ccccc}
\mathbf{L}_1 & \mathbf{A}_1 & ~ & \cdots & \mathbf{0}\\
\mathbf{A}_1^T & \mathbf{L}_2 & \mathbf{A}_2 & ~ & ~ \\
 ~ & \ddots & \ddots & \ddots  & \vdots \\
 ~ & ~ & \mathbf{A}_{m-2}^T & \mathbf{L}_{m-1} & \mathbf{A}_{m-1} \\
 \mathbf{0} & \cdots  & ~ & \mathbf{A}_{m-1}^T & \mathbf{L}_{m}
\end{array}\right],
\end{equation}
where the diagonal sub-matrices $\mathbf{L}_i\in \mathbb{R}^{l_i\times l_i}$ are sub-kernels containing DPP items that are mutually similar, and the off-diagonal sub-matrices $\mathbf{A}_{i} \in \mathbb{R}^{l_i\times l_{i+1}}$ are sparse sub-matrices with non-zero entries only at the bottom left, representing the connections between adjacent sub-kernels. Fig. \ref{fig: Syn Data} (a) gives a good example of such matrices. 

Let $\mathcal{Y}$ be the set of all indices of $\mathbf{L}$ and let $\mathcal{Y}_1,\cdots,\mathcal{Y}_m$ be that of $\mathbf{L}_1,\cdots,\mathbf{L}_m$ correspondingly. For any set of indices $C_i,C_j \subseteq \mathcal{Y}$, we use $\mathbf{L}_{C_i}$ to denote the square sub-matrix indexed by $C_i$ and $\mathbf{L}_{C_i,C_j}$ the $\vert C_i\vert \times \vert C_j\vert$ sub-matrix with rows indexed by $C_i$ and columns by $C_j$. Following general notations, by $\mathbf{L}=\diag (\mathbf{L}_1,...,\mathbf{L}_m)$ we mean the block diagonal matrix $\mathbf{L}$ consisting of sub-matrices $\mathbf{L}_1,...,\mathbf{L}_m$ and $\mathbf{L}\succeq 0$ means that $\mathbf{L}$ is positive semi-definite.
 
\section{MAP Inference for BwDPPs \label{Sec: MAP}}
\subsection{Strictly Block Diagonal Kernel}
We first consider the motivating case where the kernel is strictly block diagonal, i.e. all elements in the off-diagonal sub-matrices $\mathbf{A}_i$ are zero. It can be easily seen that the following divide-and-conquer theorem holds.
\begin{theorem}\label{thrm 0 ordr slt}
For the DPP with a block diagonal kernel $\mathbf{L}=\diag(\mathbf{L}_1,\cdots,\mathbf{L}_m)$ over ground set $\mathcal{Y}=\bigcup_{i=1}^m \mathcal{Y}_i$ which is partitioned correspondingly, the MAP solution can be obtained as: 
\begin{equation}
\hat{C} = \hat{C}_1 \cup \cdots \cup \hat{C}_m,
\end{equation}
where $\hat{C}=\displaystyle \argmax _ {C \subseteq \mathcal{Y}}\det(\mathbf{L}_C)$, and $\displaystyle\hat{C}_i=\argmax_{C_i \subseteq \mathcal{Y}_i} \det(\mathbf{L}_{C_i})$.
\end{theorem}  

Theorem \ref{thrm 0 ordr slt} tells us that the MAP inference with a strictly block diagonal kernel can be decomposed into a series of sub-inferences with its sub-kernels. In this way, the overall computation cost can be largely reduced. Noting that no exact DPP-MAP algorithms are available so far, any approximate DPP-MAP algorithms could be used in a plug-and-play way for the sub-inferences.

\subsection{Almost Block Diagonal Kernel}
Now we analyze the MAP inference for BwDPP with an almost block diagonal kernel as defined in (\ref{def: abd}). Let $C \subseteq \mathcal{Y}$ be the hypothesized subset to be selected from $\mathbf{L}$ and let $C_1 \subseteq \mathcal{Y}_1,\cdots,C_m \subseteq \mathcal{Y}_m$ be that from $\mathbf{L}_1,\cdots,\mathbf{L}_m$ correspondingly, where $C_i=C \cap \mathcal{Y}_i$. Without loss of generality, we assume $\mathbf{L}_{C_i}$ is invertible\footnote{That simply assumes that we only consider the non-trivial subsets selected with a DPP kernel $\mathbf{L}$, i.e. $\det(\mathbf{L}_{C_i})>0$. \label{Assump}} for $i=1,\cdots,m$. By defining $\mathbf{\tilde{L}}_{C_i}$ recursively as $\mathbf{\tilde{L}}_{C_i}\triangleq$
\begin{equation} \label{def: L tilde C}
\left\{
\begin{array}{cl}
\mathbf{L}_{C_i} & i=1,\\
\mathbf{L}_{C_i}-\mathbf{L}_{C_{i-1},C_i}^T\mathbf{\tilde{L}}_{C_{i-1}}^{-1}\mathbf{L}_{C_{i-1},C_i} & i=2,\cdots,m
\end{array}\right.
\end{equation}
one could rewrite the MAP objective function: $\det(\mathbf{L}_C)$
\begin{equation}\label{apr e1}
\begin{split} 
&=\det(\mathbf{L}_{C_1})\det(\mathbf{L}_{\cup_{i=2}^mC_2}-\mathbf{L}_{C_1,\cup_{i=2}^m C_i}^T\mathbf{L}_{C_1}^{-1}\mathbf{L}_{C_1,\cup_{i=2}^m C_i})\\
&=\det(\mathbf{\tilde{L}}_{C_1})\det(\begin{bmatrix}
\mathbf{\tilde{L}}_{C_2} & [\mathbf{L}_{C_2,C_3} ~ \mathbf{0}]\\ [\mathbf{L}_{C_2,C_3} ~ \mathbf{0}]^T & \mathbf{L}_{\cup_{i=3}^m C_i}
\end{bmatrix}),
\end{split}
\end{equation}
where $\mathbf{0}$ represents zero matrix of appropriate size that fill the corresponding area with zeros. The key to the second equation above is $\mathbf{L}_{C_1,C_i}=\mathbf{0}$ for $i\geq 3$, since $\mathbf{L}$ is an almost block diagonal kernel. Continuing this recursion,
\begin{equation}\label{apr e3}
\textstyle\det(\mathbf{L}_C)=\cdots=\prod_{i=1}^m \det(\mathbf{\tilde{L}}_{C_i}).
\end{equation}
Hence, the MAP objective function is reduced to:
\begin{equation}
\argmax_{C\in \mathcal{Y}}\det(\mathbf{L}_C) = \argmax _{C_1\in \mathcal{Y}_1,\cdots,C_m\in \mathcal{Y}_m}{\textstyle \prod_{i=1}^m \det(\mathbf{\tilde{L}}_{C_i})}.
\end{equation}

As $\mathbf{\tilde{L}}_{C_i}$ depends on $C_1,\cdots,C_i$, we cannot optimize $\det(\mathbf{\tilde{L}}_{C_1}),\cdots,\det(\mathbf{\tilde{L}}_{C_m})$ separately. Alternatively, we provide an approximate method that optimize over $C_1,\cdots,C_m$ sequentially, named the BwDPP-MAP method, which is a depth-first greedy search method in essence. The BwDPP-MAP is described in Table \ref{BwDPP-MAP Alg}, where $\argmax_{C_i;C_j=\hat{C}_j, j=1,\cdots,i-1}$ denotes optimizing over $C_i$ with the value of $C_j$ fixed as $\hat{C}_j$ for $j=1,\cdots,i-1$, and the sub-kernel\footnote{Both $\mathbf{L}_{\mathcal{Y}_i}$ and $\mathbf{\tilde{L}}_{\mathcal{Y}_i}$ are called sub-kernels.} $\mathbf{\tilde{L}}_{\mathcal{Y}_i}$ is given similarly as $\mathbf{\tilde{L}}_{C_i}$, namely $\mathbf{\tilde{L}}_{\mathcal{Y}_i}\triangleq$
\begin{equation}\label{apr e4}
\left\{
\begin{array}{cl}
\mathbf{L}_{i} & i=1,\\
\mathbf{L}_{i}-\mathbf{L}_{C_{i-1},\mathcal{Y}_i}^T\mathbf{\tilde{L}}_{C_{i-1}}^{-1}\mathbf{L}_{C_{i-1},\mathcal{Y}_i} & i=2,\cdots,m
\end{array}\right.
\end{equation}
One may notice that $(\mathbf{\tilde{L}}_{\mathcal{Y}_i})_{C_i}$ is equivalent to $\mathbf{\tilde{L}}_{C_i}$.

\begin{table}
\caption{BwDPP-MAP Algorithm}
\label{BwDPP-MAP Alg} 
\centering
\begin{tabular}{l}
\toprule[1pt]
{\bf Input:} \hspace{0.5em} $\mathbf{L}$ as defined in (\ref{def: abd}); \\
{\bf Output:} \hspace{0.5em} Subset of items $\hat{C}$.\\
\hline
{\bf For:} $i = 1,\cdots, m$\\

\hspace{1.5em} Compute $\mathbf{\tilde{L}}_{\mathcal{Y}_i}$ via (\ref{apr e4});\\ 
\hspace{1.5em} Perform sub-inference over $C_i$ via \\
\hspace{1.5em} $ \hat{C}_i=\argmax_{C_i\in \mathcal{Y}_i;C_j=\hat{C}_j, j=1,\cdots,i-1}\det((\mathbf{\tilde{L}}_{\mathcal{Y}_i})_{C_i})$;\\

{\bf Return:} $\hat{C}=\bigcup_{i=1}^m \hat{C}_i$. \\
\bottomrule[1pt]
\end{tabular}
\end{table}

In conclusion, similar to the MAP inference with a strictly block diagonal kernel, by using BwDPP-MAP, the MAP inference for an almost block diagonal kernel can be decomposed into a series of sub-inferences for the sub-kernels as well. There are four comments for this conclusion.

First, it should be noted that the above BwDPP-MAP method is an approximate optimization method, even if each sub-inference step is conducted exactly. This is because $\mathbf{\tilde{L}}_{C_i}$ depends on $C_1,\cdots,C_i$. We provide an empirical evaluation later, showing that through block-wise operation, the greedy search in BwDPP-MAP can achieve computation speed-up with marginal sacrifice of the accuracy.

Second, by the following Lemma \ref{lm apr}, we show that each sub-kernel $\mathbf{\tilde{L}}_{\mathcal{Y}_i}$ is positive semi-definite, so that it is theoretically guaranteed that we can conduct each sub-inference via existing DPP-MAP algorithms, e.g. the greedy DPP-MAP algorithm (Table \ref{Greedy MAP Alg}) \cite{gillenwater2012near}. One may find the proof of Lemma \ref{lm apr} in the appendix. 

\begin{lemma}\label{lm apr}
$\mathbf{\tilde{L}}_{\mathcal{Y}_i}\succeq 0$, for $i=1,\cdots,m$.
\end{lemma}

Third, in order to apply BwDPP-MAP, we need to first partition a given DPP kernel into the form of an almost block diagonal matrix as defined in (\ref{def: abd}). The partition is not unique. A trivial partition for an arbitrary DPP kernel is no partition, i.e., regarding the whole matrix as a single block. We leave the study of finding the optimal partition for further work. Here we provide a heuristic rule for partition, which is called $\gamma$-partition and performs well in our experiments.

\begin{definition}\label{def: conlvl}
($\gamma$-partition) A $\gamma$-partition is defined by partitioning a DPP kernel $\mathbf{L}$ into the almost block diagonal form as defined in (\ref{def: abd}) with the maximum number of blocks (i.e. the largest possible m)\footnote{Generally speaking, a partition of a kernel of size $N$ into $m$ sub-kernels will approximately reduce the computational complexity $m^2$ times. A larger $m$ implies larger computation reduction.}, where for every off-diagonal matrix $\mathbf{A}_i$, the size of its non-zero area is only at the bottom left and does not exceed $\gamma \times \gamma$.
\end{definition}

A heuristic way to obtain $\gamma$-partition for a kernel L is to first identify a series of non-overlapping dense square sub-matrices along the main diagonal as many as possible. Next, two adjacent square sub-matrices in the main diagonal are merged if the size of the non-zero area in their corresponding off-diagonal sub-matrix exceeds $\gamma \times \gamma$.

It should be noted that a kernel could be subject to $\gamma$-partition in one or more ways with different values of $\gamma$. By taking $\gamma$-partition for a kernel with different values of $\gamma$, we can obtain a balance between computation cost and optimization accuracy. A smaller $\gamma$ implies smaller $m$ achievable in $\gamma$-partition, and thus smaller computation reduction. On the other hand, a smaller $\gamma$ means smaller degree of interaction between adjacent sub-inferences, and thus better optimization accuracy.

Fourth, an empirical illustration of BwDPP-MAP is given in Fig. \ref{fig: Syn Data}, where the greedy MAP algorithm (Table \ref{Greedy MAP Alg}) \cite{gillenwater2012near} is used for the sub-inferences in BwDPP-MAP. The synthetic kernel size is fixed as $500$. For each realization, the area of non-zero entries in the kernel is first specified by uniformly randomly choosing the size of sub-kernels from $[10, 30]$ and the size of the non-zero areas in off-diagonal sub-matrices from $\{0,2,4,6\}$. Next, a vector $\mathbf{B}_i$ is generated for each item $i$ separately, following standard normal distribution. Finally, for all non-zero entries ($L_{ij}\neq 0$) specified in the previous step, the entry value is given by $L_{ij}=\mathbf{B}_i^T \mathbf{B}_j$. Fig. \ref{fig: Syn Data} (a) provides an example for such synthetic kernel. 

We generate 1000 synthetic kernels as described above. For each synthetic kernel, we take $\gamma$-partition with $\gamma=0,2,4,6$, and then run BwDPP-MAP. The performance of directly applying the greedy MAP algorithm on the original unpartitioned kernel is used as baseline. The results in Fig. \ref{fig: Syn Data} (b) show that BwDPP-MAP runs much faster than the baseline. With the increase of $\gamma$, the runtime drops while the inference accuracy degrades within a tolerable range.  

\subsection{Connection between BwDPP-MAP and its Sub-inference Algorithm}
Any DPP-MAP inference algorithm can be used in a plug-and-play fashion for the sub-inference procedure of BwDPP. It is natural to ask the connection between BwDPP-MAP and its corresponding DPP-MAP algorithm. The relation is given by the following result.

\begin{theorem} \label{thrm connection}
Let $f$ be any DPP-MAP algorithm for BwDPP-MAP sub-inference, where $f$ maps a positive semi-definite matrix to a subset of its indices, i.e. $f: \mathbf{L} \in \mathbb{S}_+ \mapsto Y \subseteq \mathcal{Y}$. BwDPP-MAP (table \ref{BwDPP-MAP Alg}) is equivalent to applying the following steps successively to the almost block diagonal kernel as defined in (\ref{def: abd}):
\begin{equation}
\hat{C}_1 = f(\mathbf{L}_{\mathcal{Y}_1}),
\end{equation}
and for $i=2,...,m$,
\begin{equation} \label{eq conditional}
\hat{C}_i = f(\mathbf{L}_{\cup_{j=1}^i \mathcal{Y}_j} \vert \hat{C}_{1:i-1} \subseteq Y, \bar{\hat{C}}_{1:i-1}\cap Y=\emptyset).
\end{equation}
where $\hat{C}_{1:i-1} = \cup_{j=1}^{i-1} \hat{C}_j$, $\bar{\hat{C}}_{1:i-1}=\cup_{j=1}^{i-1} (\mathcal{Y}_i/\hat{C}_j)$, and the input of $f$ is the conditional kernel\footnote{The conditional distribution (over set $\mathcal{Y}-A^{in}-A^{out}$) of the DPP defined by $\mathbf{L}$,
\begin{equation}
\mathcal{P}_{\mathbf{L}} (Y=A^{in} \cup B \vert A^{in} \subseteq Y, A^{out}\cap Y=\emptyset),
\end{equation} is also a DPP \cite{kulesza2012determinantal}, and the corresponding kernel, $\left(\mathbf{L}\vert A^{in}\subseteq Y, A^{out}\cap Y=\emptyset\right)$, is called the conditional kernel.}.

\end{theorem}

The proof of Theorem \ref{thrm connection} is in the appendix. Theorem \ref{thrm connection} states that BwDPP-MAP is essentially a series of Bayesian belief updates, where in each update a conditional kernel is fed into $f$ that contains the information of previous selection result. The equivalent form allows us to compare BwDPP-MAP directly with the method of applying $f$ on the entire kernel. The latter does inference on the entire set $\mathcal{Y}$ for one time, while the former does the inference on a sequence of smaller subsets $\mathcal{Y}_1,...,\mathcal{Y}_m$. Concretely, in the $i$-th update, a subset $\mathcal{Y}_i$ is added to have the kernel $L_{\cup_{j=1}^i \mathcal{Y}_j}$. Then the information of previous selection result is incorporated into the kernel to generate the conditional kernel. Finally, the DPP-MAP inference is performed on the conditional kernel to select $\hat{C}_i$ from $\mathcal{Y}_i$. 

\begin{figure}
\centering
\subfigure[~]{\includegraphics[width=3cm]{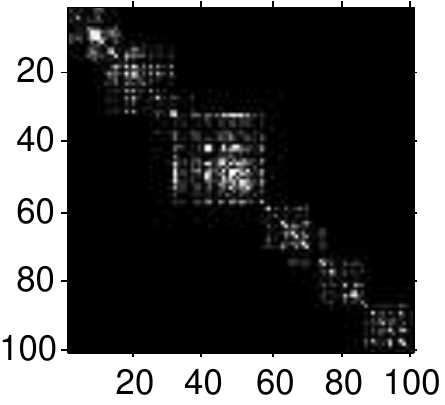}}~~~~
\subfigure[~]{\includegraphics[width=3.6cm]{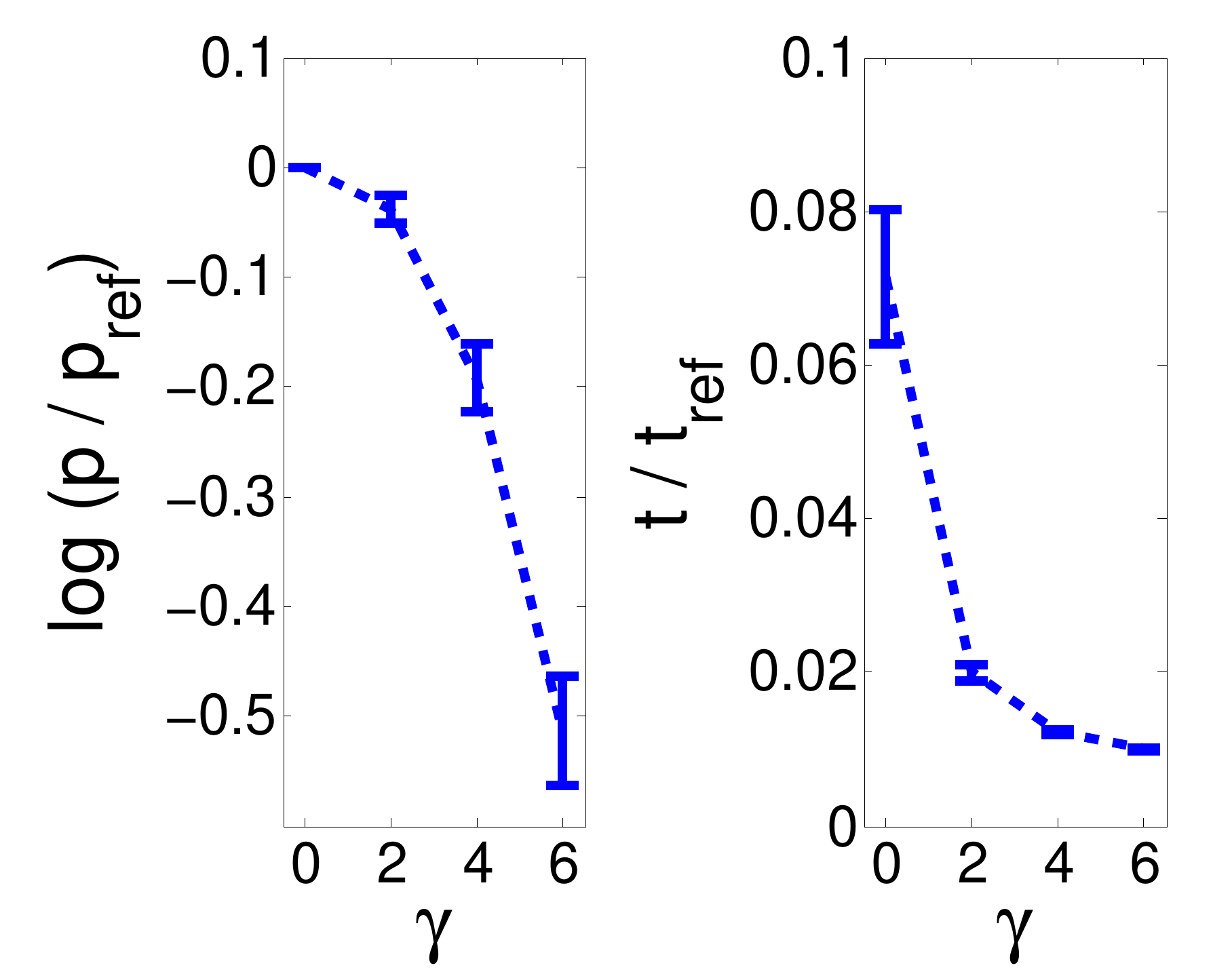}}
\caption{(a) The top-left $100 \times 100$ entries from a $500\times 500$ synthetic kernel. (b) The log-probability ratio $\log(p/p_{\rm{ref}})$ and runtime ratio $t/t_{\rm{ref}}$, obtained from using BwDPP-MAP on the same kernel with different $\gamma$-partition, where $p_{\rm{ref}}$ and $t_{\rm{ref}}$ are the baseline performance of directly applying the greedy MAP algorithm on the original unpartitioned kernel. Results are averaged over $1000$ kernels. The error bar represents $99.7\%$ confidence level.}
\label{fig: Syn Data}
\end{figure}

\begin{table}
\caption{Greedy DPP-MAP Algorithm}
\label{Greedy MAP Alg} 
\centering
\begin{tabular}{l}
\toprule[1pt]
{\bf Input:} \hspace{0.5em} $\mathbf{L}$; \hspace{1.5em} {\bf Output:} \hspace{0.5em} $\hat{C}$.\\
\hline
{\bf Initialization:} \hspace{0.5em} Set $\hat{C}\leftarrow\emptyset$, $U\leftarrow\mathcal{Y}$;\\
{\bf While} $U$ is not empty;\\
\hspace{1.5em} $i^*\leftarrow\argmax _{i \in U} L_{ii}$; \hspace{1.5em} $\hat{C}\leftarrow \hat{C}\cup\{i^*\}$;\\
\hspace{1.5em} Compute $\mathbf{L}^*=\left(\left[(\mathbf{L}+\mathbf{I}_{\bar{\hat{C}}})^{-1}\right]_{\bar{\hat{C}}}\right)^{-1}-\mathbf{I}$; \\
\hspace{1.5em} $\mathbf{L}\leftarrow \mathbf{L}^*$; \hspace{1.5em} $U\leftarrow \{i\vert i \notin \hat{C}, \mathbf{L}_{ii}>1\}$;\\
{\bf Return:} $\hat{C}$. \\
\bottomrule[1pt]
\end{tabular}
\end{table} 
\section{BwDPP-based Change-Point Detection \label{Sec: BwDPP CPD}}
Let $\mathbf{x}_1,\cdots,\mathbf{x}_T$ be the time-series observations, where $\mathbf{x}_t \in \mathbb{R}^D$ represents the $D$-dimensional observation at time $t=1,\cdots,T$, and let $\mathbf{x}_{\tau:t}$ denote the segment of observations in the time interval $[\tau ,t]$. We further use $\mathbf{X}_1$, $\mathbf{X}_2$ to represent different segments of observations at different intervals, when explicitly denoting the beginning and ending times of the intervals are not necessary. The new CPD method will build on existing metrics. A dissimilarity metric is denoted as $d: (\mathbf{X}_1, \mathbf{X}_2) \mapsto \mathbb{R}$, which measures the dissimilarity between two arbitrary segments $\mathbf{X}_1$ and $\mathbf{X}_2$.

\subsection{Quality-Diversity Decomposition of Kernel}
Given a set of items $\mathcal{Y}=\{1,\cdots,N\}$, the DPP kernel $\mathbf{L}$ can be written as a Gram matrix $\mathbf{L}=\mathbf{B}^T\mathbf{B}$, where $\mathbf{B}_i$, the columns of $\mathbf{B}$, are vectors representing items in $\mathcal{Y}$. 

A popular decomposition of the kernel is to define $\mathbf{B}_i=q_i \bm{\phi}_i$, where $q_i\in \mathbb{R}^+$ measures the quality (magnitude) of item $i$ in $\mathcal{Y}$, and $\bm{\phi}_i\in \mathbb{R}^k$, $\Vert \bm{\phi}_i\Vert=1$ can be viewed as the angle vector of diversity features so that $\bm{\phi}_i^T \bm{\phi}_j$ measures the similarity between items $i$ and $j$. Therefore, $\mathbf{L}$ is defined as
\begin{equation}
\mathbf{L}=\diag (\mathbf{q}) * \mathbf{S} * \diag(\mathbf{q}),
\end{equation}
where $\mathbf{q}$ is the quality vector consisting of $q_i$, and $\mathbf{S}$ is the similarity matrix consisting of $S_{ij}=\bm{\phi}_i^T \bm{\phi}_j$. The quality-diversity decomposition allows us to construct $\mathbf{q}$ and $\mathbf{S}$ separately to address different concerns, which is utilized below to construct the kernel for CPD.

\subsection{BwDppCpd}
BwDppCpd is a two-step CPD method, described as follows.

\textbf{Step 1:} Based on a dissimilarity metric $d$, a preliminary set of change-point candidates is created. Consider moving a pair of adjacent windows, $\mathbf{x}_{t-w+1:t}$ and $\mathbf{x}_{t+1:t+w}$, along $t=w,\cdots,T-w$, where $w$ is the size of local windows. Then, a large $d$ value for the adjacent windows, i.e. $d(\mathbf{x}_{t-w+1:t},\mathbf{x}_{t+1:t+w})$, suggests that a change-point is likely to occur at time t. After we obtain the series of $d$ values, local peaks above the mean of the $d$ values are marked and the corresponding locations, say $t_1,\cdots,t_N$, are selected to form the preliminary set of change-point candidates $\mathcal{Y}=\{1,\cdots,N\}$.

\textbf{Step 2:} Treat the change-point candidates $\mathcal{Y}=\{1,\cdots,N\}$ as BwDPP items, and select a subset from them to be our final estimate of the change-points.

The BwDPP kernel is built via quality-diversity decomposition. We use the similarity metric $d$ once more to measure the quality of a candidate change-point to be a true one. Specifically, we define
\begin{equation}
q_i = d (\mathbf{x}_{t_{i-1}:t_i}, \mathbf{x}_{t_i:t_{i+1}}),
\end{equation}
The higher the value $q_i$ is, the sharper contrast around the change-point candidate $i$, and the better quality of $i$.

Next, the BwDPP similarity matrix is defined to address the fact that the true locations of change-points along the timeline tend to be diverse, since states would not change rapidly. This is done by assigning high similarity score to items being close to each other. Specifically, we define
\begin{equation}
S_{ij}=\exp ({-{(t_i-t_j)^2}/{\sigma^2}}),
\end{equation}
where $\sigma$ is a parameter representing the position diversity level. Finally, after taking $\gamma$-partition of the kernel $\mathbf{L}$ into the almost block diagonal form, BwDPP-MAP is used to select a set of change-points that favours both quality and diversity (Fig. \ref{fig: Hasc Demo} (b)).

\subsection{Discussion}
There is a rich studies of metrics for CPD problem. The choice of the dissimilarity metric $d(\mathbf{X}_1,\mathbf{X}_2)$ is flexible and could be well-tailored to the characteristics of the data. We present two examples that are used in our experiments.

\begin{itemize}
\item Symmetric Kullback-Leibler Divergence (SymKL):\\
If the two segments $\mathbf{X}_{1}$,$\mathbf{X}_{2}$ to be compared are assumed to follow Gaussian processes, the SymKL metric is given:
\begin{equation}\label{eq SymKL}
\begin{split}
&{\rm{SymKL}} (\mathbf{X}_{1},\mathbf{X}_{2})=\tr (\bm{\Sigma}_{1} \bm{\Sigma}_{2}^{-1}) + \tr(\bm{\Sigma}_{2} \bm{\Sigma}_{1}^{-1}) -\\ &2D + \tr((\bm{\Sigma}_{1}^{-1}+\bm{\Sigma}_{2}^{-1})(\bm{\mu}_{1}-\bm{\mu}_{2})(\bm{\mu}_{1}-\bm{\mu}_{2})^T),
\end{split}
\end{equation} 
where $\bm{\mu}$ and $\bm{\Sigma}$ are corresponding sample mean and covariance.
\item Generalized Likelihood Ratio (GLR):\\
Generally, the GLR metric is given by the likelihood ratio:
\begin{equation}
{\rm{GLR(\mathbf{X}_1,\mathbf{X}_2)}}=\frac{\mathcal{L}(\mathbf{X}_1\vert \lambda_1)\mathcal{L}(\mathbf{X}_2\vert \lambda_2)}{\mathcal{L}(\mathbf{X}_{1,2}\vert \lambda_{1,2})}.
\end{equation}
The numerator is the likelihood that the two segments follows two different models $\lambda_1$ and $\lambda_2$ respectively, while the denominator is that two segments together (denoted as $\mathbf{X}_{1,2}$) follows a single model $\lambda_{1,2}$. In practice, we plug the maximium likelihood estimates (MLE) for the parameters $\lambda_1$, $\lambda_2$, and $\lambda_{1,2}$. E.g. if we assume that the time-series segment $\mathbf{X}\triangleq \{x_1,\cdots,x_M\}$ follows a homogeneous Poisson process, where $x_i$ is the occurring time of the $i$-th event, $i=1,\cdots,M$. The log-likehood of $\mathbf{X}$ is
\begin{equation}
\mathcal{L}(\mathbf{X}\vert \lambda) = (M-1) \log\lambda - (x_M-x_1) \lambda 
\end{equation} 
where the MLE of $\lambda$ is used, $\lambda= (M-1)/ (x_M-x_1)$.
\end{itemize}

\begin{figure}
\centering
\subfigure[~]{\includegraphics[width=4.1cm]{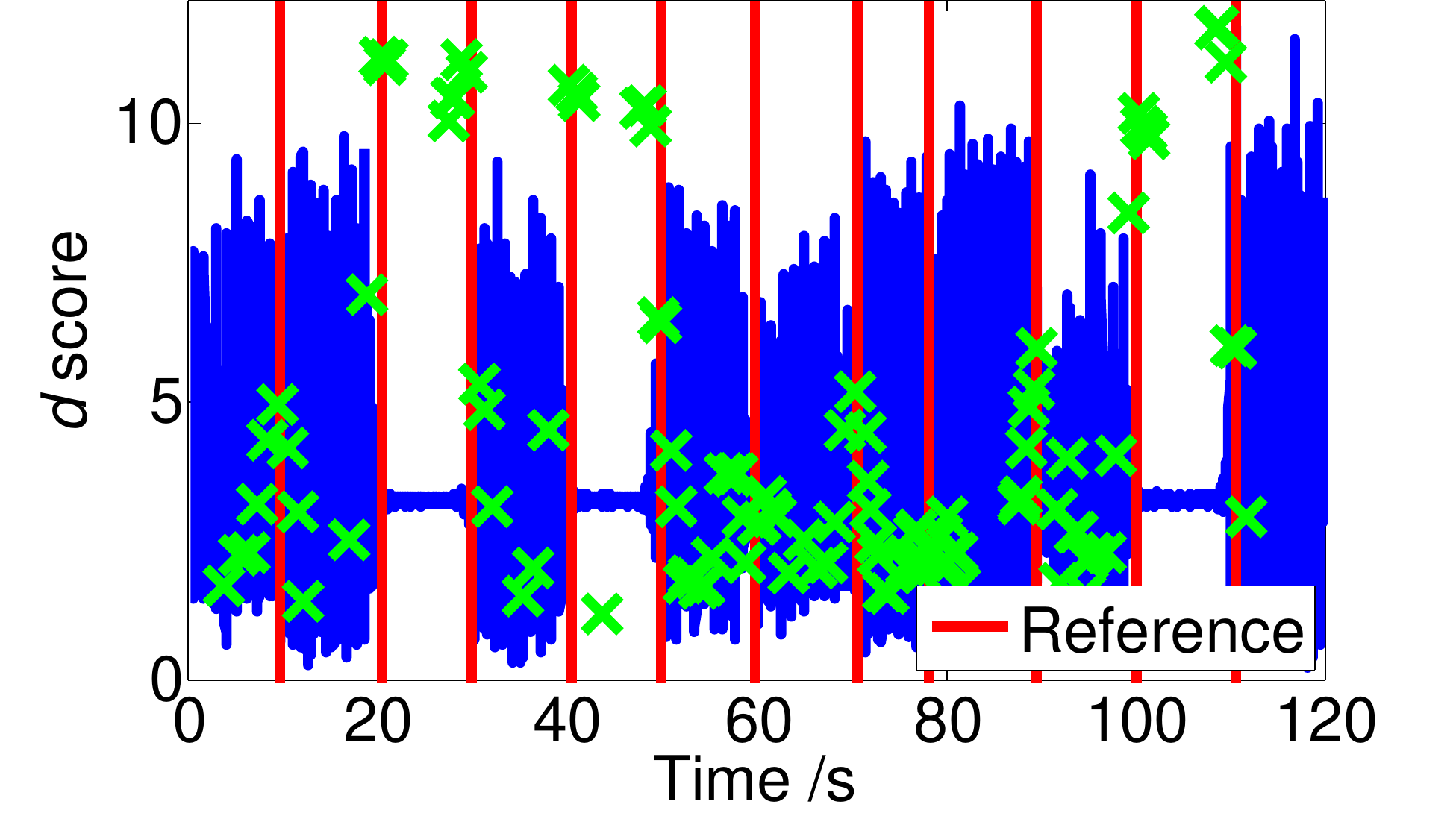}}
\subfigure[~]{\includegraphics[width=4.1cm]{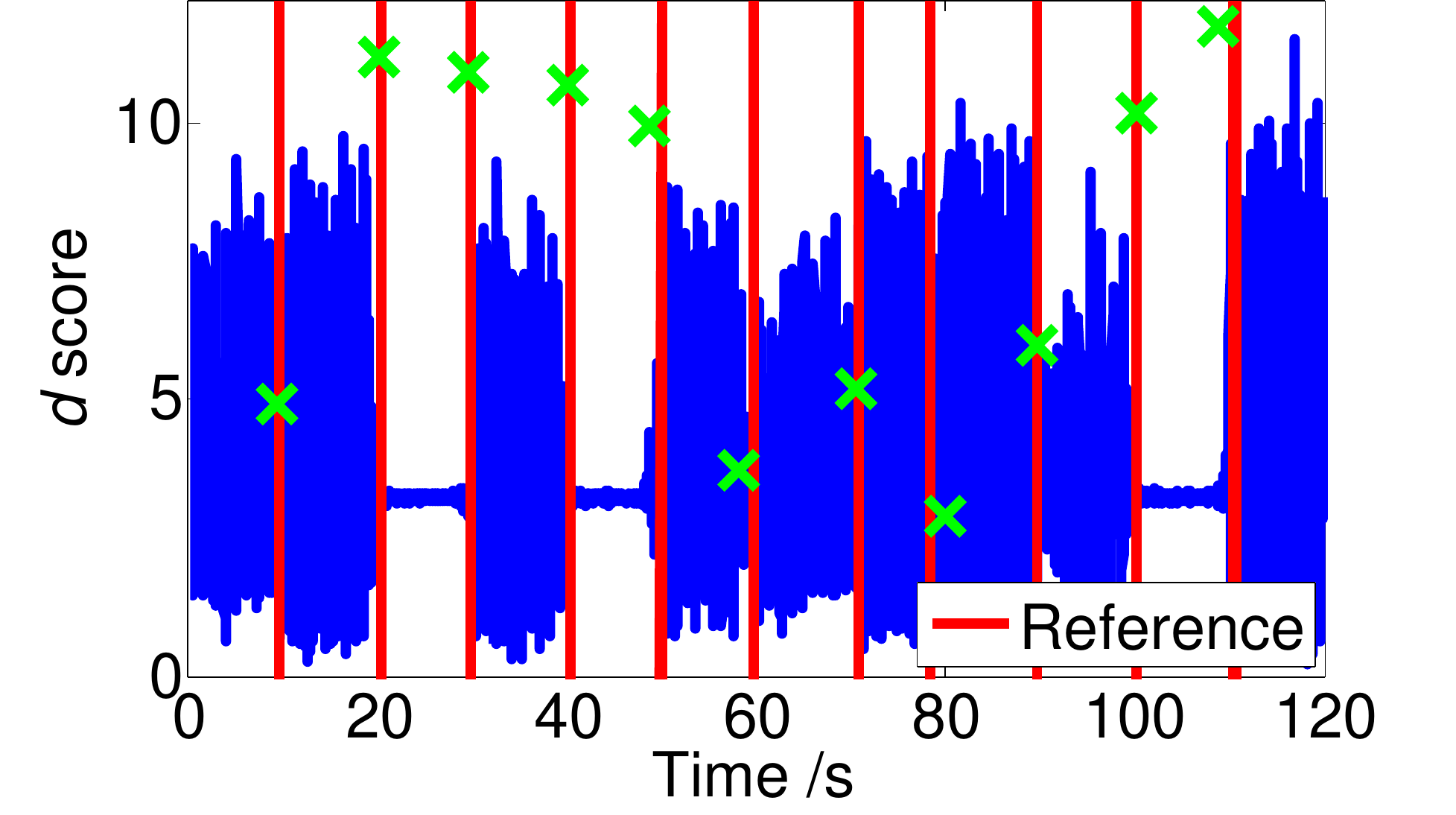}}
\qquad
\caption{An BwDppCpd example from \emph{Hasc}. (a) Change-point candidates selected in Step 1 with their $d$ scores (green cross). (b) Final estimate of change-points in step 2 with their $d$ scores (green cross).}
\label{fig: Hasc Demo}
\end{figure}
\section{Experiments \label{Sec: Exp}}
\begin{figure}
\centering
\subfigure[~]{\includegraphics[width=4.2cm]{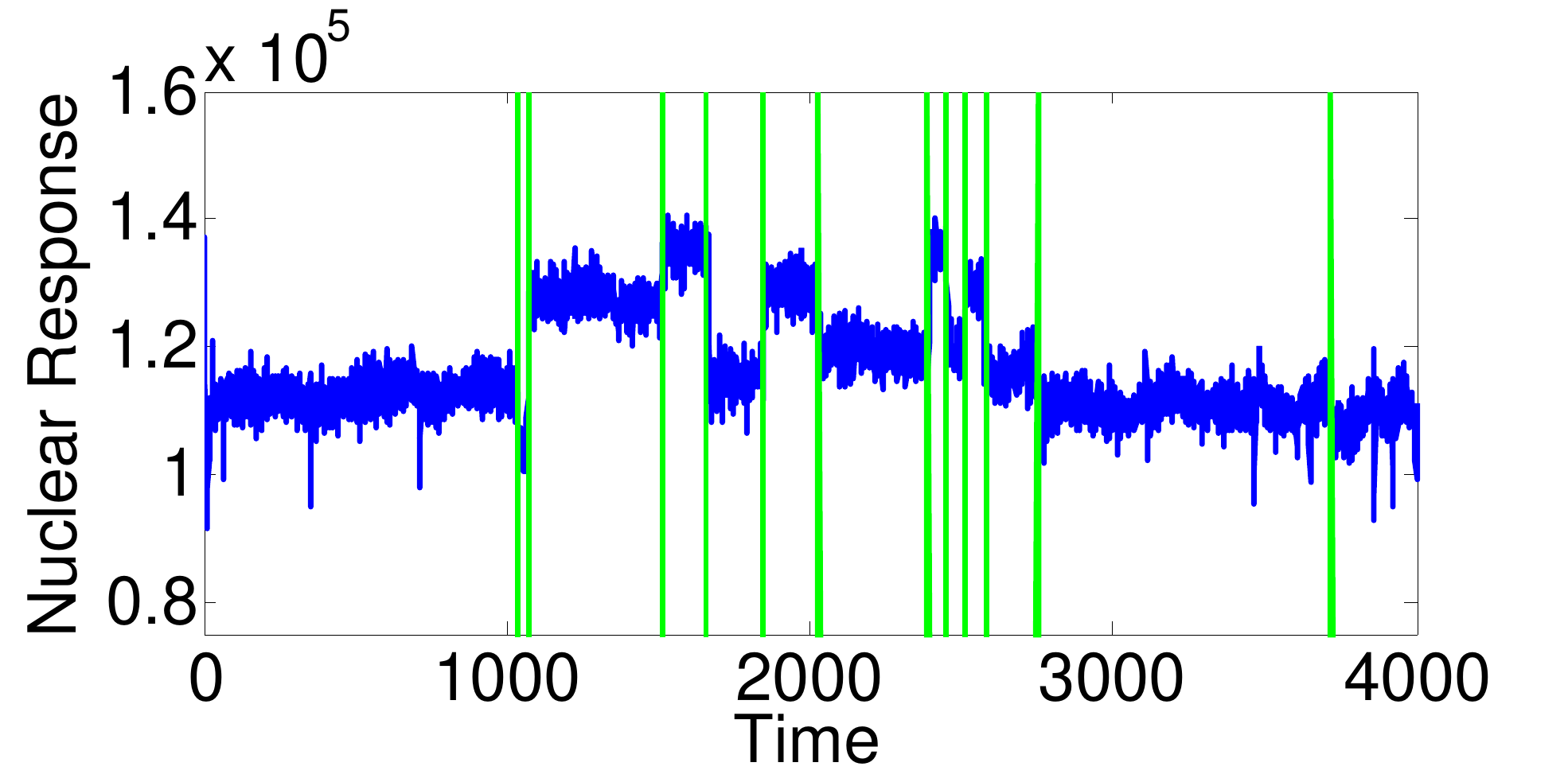}}
\subfigure[~]{\includegraphics[width=4cm]{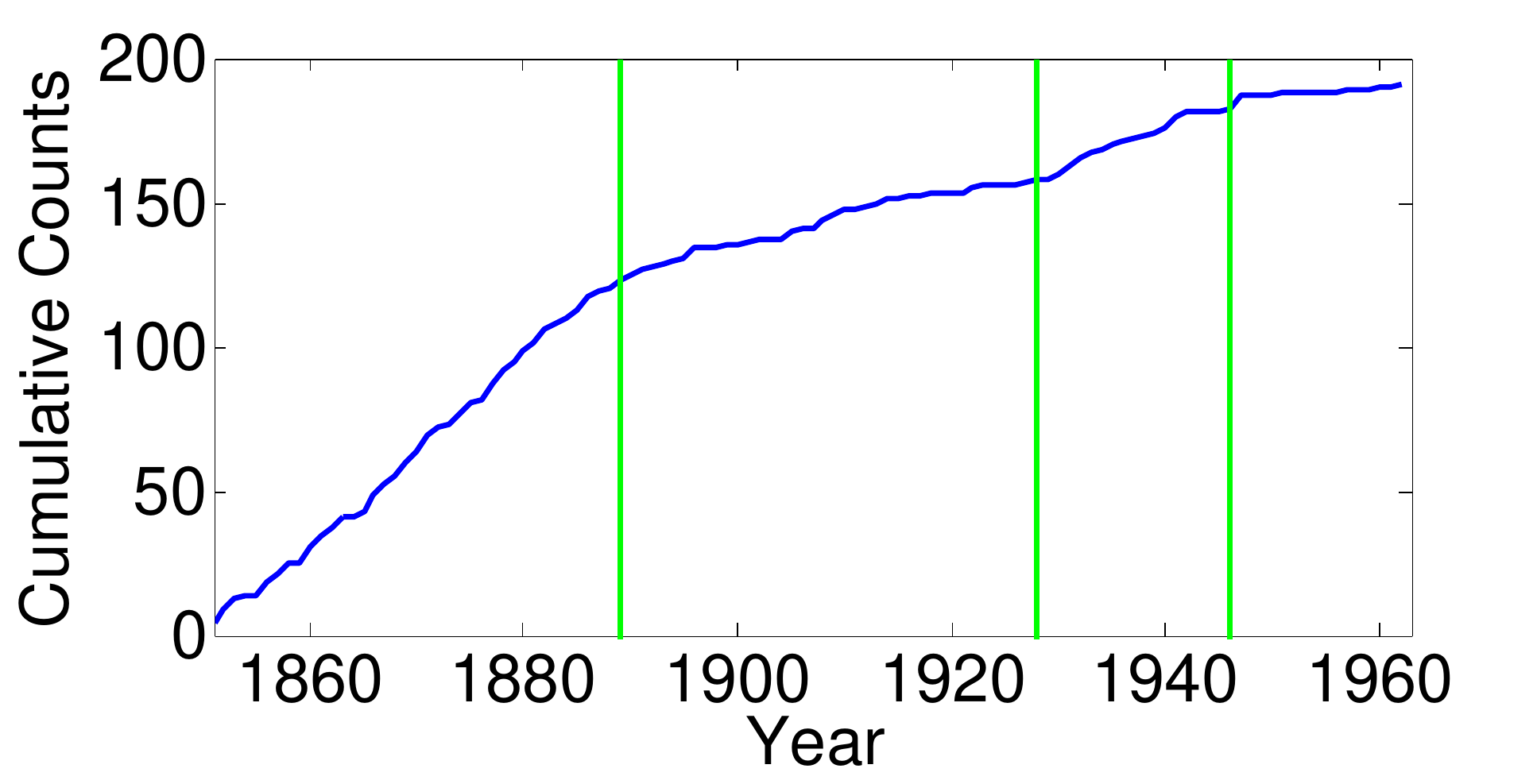}}
\qquad
\subfigure[~]{\includegraphics[width=6cm]{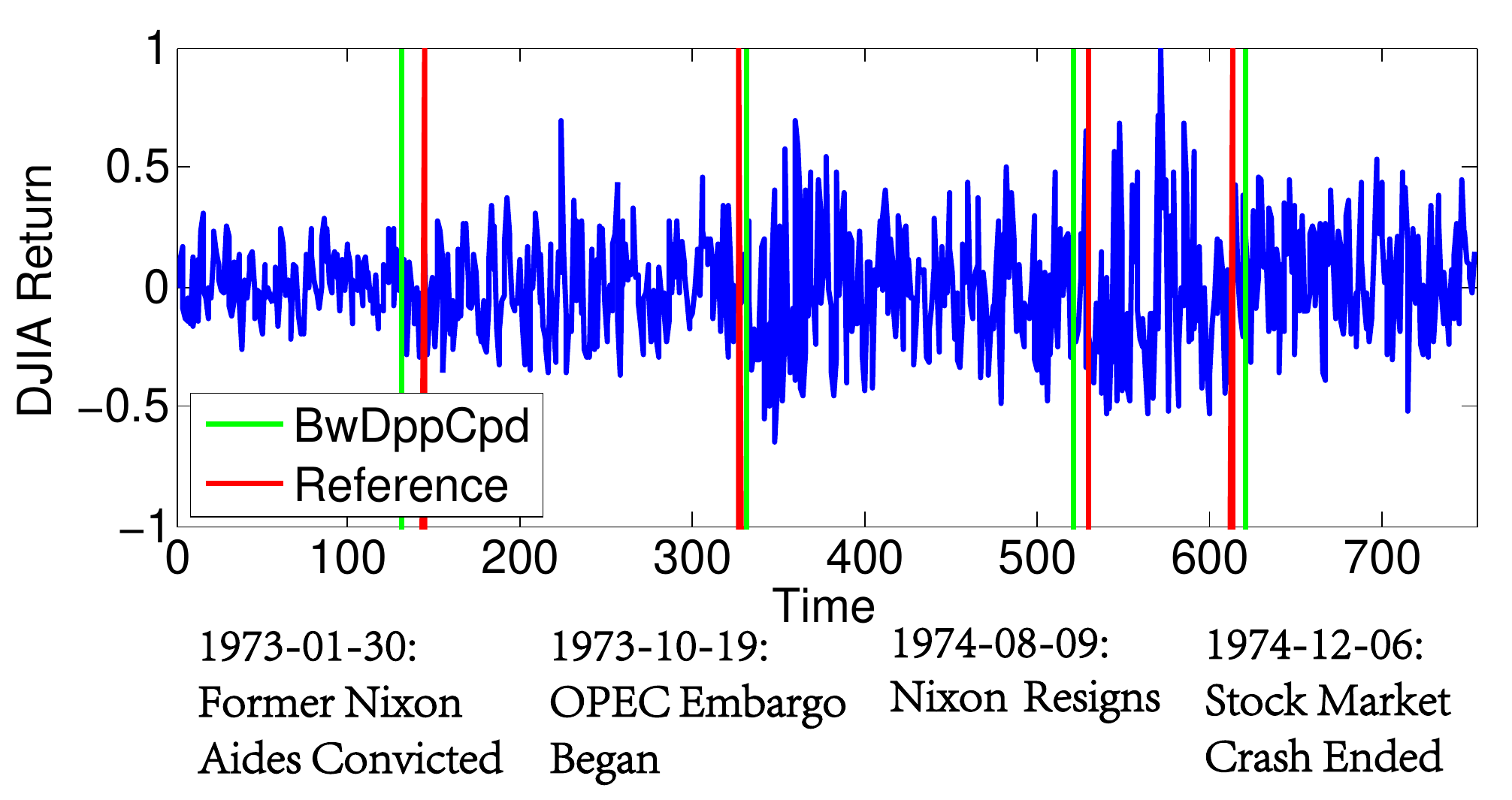}}
\caption{BwDppCpd results for \emph{Well-Log} (a), \emph{Coal Mine Disaster} (b), and \emph{DJIA} (c). Green lines are detected changes.}
\label{fig: real-world Data}
\end{figure}

The BwDppCpd method are evaluated on five real-world time-series data. Firstly, three classic datasets are examined for CPD, namely \emph{Well-Log} data, \emph{Coal Mine Disaster} data, and \emph{Dow Jones Industrial Average Return (DJIA)} data, where we set $\gamma=0$ due to the small data size. 

Next, we experiment with human activity detection and speech segmentation, where the data size becomes larger and there is no accurate model to characterize the data, making the CPD task harder. In both experiments, the numbers of DPP items varies from hundreds to thousands, where, except BwDPP-MAP, no other algorithms can perform MAP inference within a reasonable cost of time due to the large kernel scale. We set $\gamma=3$ for human activity detection and $\gamma=0,2$ for speech segmentation to provide a comparison. 

As for the dissimilarity metric $d$, Poisson processes and GLR are used in \emph{Coal Mine Disaster} and for other experiments, Gaussian models and SymKL are used.

\subsection{Well-Log Data}
\emph{Well-Log} contains 4050 measurements of nuclear magnetic response taken during the drilling of a well. It is an example of varying Gaussian mean and the changes reflect the stratification of the earth's crust \cite{adams2007bayesian}. Outliers are removed prior to the experiment. As shown in Fig. \ref{fig: real-world Data} (a), all changes are detected by BwDppCpd. 

\subsection{Coal Mine Disaster Data}
\emph{Coal Mine Disaster} \cite{jarrett1979note}, a standard dataset for testing CPD method, consists of 191 accidents from 1851 to 1962. The occurring rates of accidents are believed to have changed a few times and the task is to detect them. The BwDppCpd detection result, as shown in Fig. \ref{fig: real-world Data} (b), agrees with that in \cite{green1995reversible}.

\subsection{1972-75 Dow Jones Industrial Average Return}
\emph{DJIA} contains daily return rates of Dow Jones Industrial Average from 1972 to 1975.
It is an example of varying Gaussian variance, where the changes are caused by big events that have potential macroeconomic effects. Four changes in the data are detected by BwDppCpd, which are matched well with important events (Fig. \ref{fig: real-world Data} (c)). Compared to \cite{adams2007bayesian}, one more change is detected (the rightmost), which corresponds to the date that 73-74 stock market crash ended\footnote{http://en.wikipedia.org/wiki/1973-74\_stock\_market\_crash}. This shows that the BwDppCpd discovers more information from the data. 

\begin{table}
\centering
\begin{tabular}{|c|c|c|c|}
\hline 
~ & {\rm{PRC}}$\%$ & {\rm{RCL}}$\%$ & $F_1$ \\
\hline 
BwDppCpd & 93.05 & 87.88 & 0.9039 \\
\hline 
RuLSIF & 86.36 & 83.84 & 0.8508 \\
\hline 
\end{tabular}
\caption{CPD result on human activity detection data \emph{HASC}. \label{tab: Hasc}}
\end{table}

\begin{figure}
\centering
\includegraphics[width=5cm]{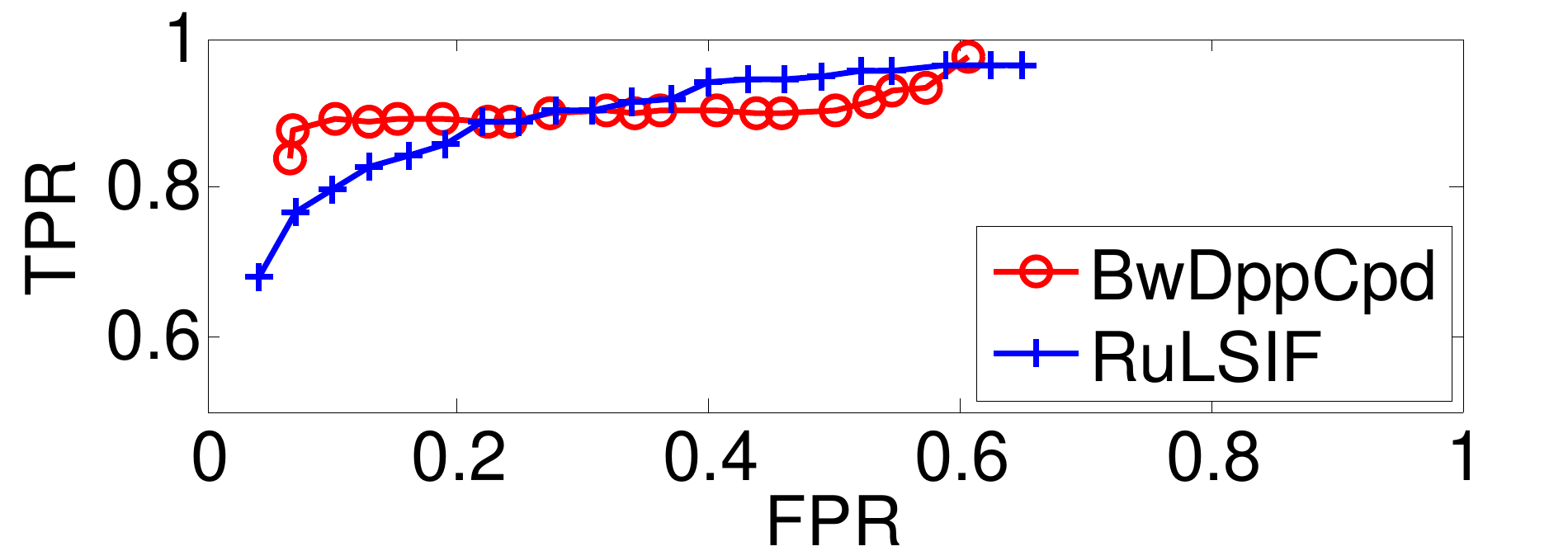}
\caption{The ROC curve of BwDppCpd and RuLISF.}
\label{fig: Roc}
\end{figure}

\subsection{Human Activity Detection}
\emph{HASC}\footnote{http://hasc.jp/hc2011/} contains human activity data collected by portable three-axis accelerometers and the task is to segment the data according to human behaviour changes. Fig. \ref{fig: Hasc Demo} (b) shows an example of \emph{Hasc}. The performance of the best algorithm in \cite{liu2013change}, RuLSIF, is used for comparison and the precision (PRC), recall (RCL), and $F_1$ measure \cite{kotti2008speaker} are used for evaluation:
\begin{align}
&{\rm{PRC}}={{\rm{CFC}}}/{{\rm{DET}}}, ~~~~{\rm{RCL}}={{\rm{CFC}}}/{{\rm{GT}}},\\
&{F_1}=2~{\rm{PRC}}~{\rm{RCL}}/({{\rm{PRC}}+{\rm{RCL}}}),
\end{align}
where ${\rm{CFC}}$ is the number of correctly found changes, ${\rm{DET}}$ is the number of detected changes, and ${\rm{GT}}$ is the number of ground-truth changes. $F_1$ score could be viewed as a overall score that balances PRL and RCL. The CPD result is shown in Table \ref{tab: Hasc}, where the parameters are set to attain the best $F_{\rm{1}}$ results for both algorithms.

The receiver operating characteristic (ROC) curve is often used to evaluate performance under different precision and recall, where true positive rate (TPR) and false positive rate (FPR) are given by ${\rm{TPR}}={\rm{RCL}}$ and ${\rm{FPR}}=1-{\rm{PRC}}$. For BwDppCpd, different levels of TPR and FPR are obtained by tuning the position diversity parameter $\sigma$ and for RuLSIF by tuning the threshold $\eta$ \cite{liu2013change}.

As shown in Table \ref{tab: Hasc} and Fig. \ref{fig: Roc}, BwDppCpd outperforms RuLISF on \emph{HASC} when the FPR is low. RuLISF has a better performance only when FPR exceeds $0.3$, which is less useful.

\subsection{Speech Segmentation}
We tested two datasets for speech segmentation. The first dataset, called \emph{Hub4m97}, is a subset (around 5 hours) from 1997 Mandarin Broadcast News Speech (HUB4-NE) released by LDC\footnote{http://catalog.ldc.upenn.edu/LDC98S73}. The second dataset, called \emph{TelRecord}, consists of 216 telephone conversations, each around 2-min long, collected from real-world call centres. Acoustic features of 12-order MFCCs (mel-frequency cepstral coefficients) are extracted as the time-series data. 

Speech segmentation is to segment the audio data into acoustically homogeneous segments, e.g. utterances from a single speaker or non-speech portions. The two datasets contain utterances with hesitations and a variety of changing background noises, presenting a great challenge for CPD.

The  BwDppCpd method with different $\gamma$ for kernel partition (denoted as Bw-$\gamma$ in Table \ref{tab SegResult}) is tested and two classic segmentation methods BIC \cite{chen1998speaker} and DISTBIC \cite{delacourt2000distbic} are used for comparison. As the same as in (Delacourt and Wellekens 2000), a post-processing step based on BIC values is also taken to reduce the false alarms for BwDppCpd.

The experiment results in Table \ref{tab SegResult} shows that BwDppCpd outperforms BIC and DISTBIC for both datasets. In addition, comparing the results obtained with $\gamma=0$ and $\gamma=2$, using $\gamma=2$ is found to be faster but has a slightly worse performance. This agrees with our analysis of BwDPP-MAP for using different $\gamma$-partition to tradeoff speed and accuracy.

\begin{table}
\centering
\begin{tabular}{|c|c|c|c|c|}
\hline 
~ & BIC & DistBIC & Bw-$0$ & Bw-$2$\\
\hline 
\multicolumn{5}{|c|}{\emph{Hub4m97}}  \\
\hline 
{\rm{PRC}$\%$} & 59.40 & 64.29 & 65.29 & 65.12\\
\hline 
{\rm{RCL}$\%$} & 78.24 & 74.98 & 78.49 & 78.39\\
\hline 
$F_1$ & 0.6753 & 0.6922 & 0.7128 & 0.7114\\
\hline 
\multicolumn{5}{|c|}{\emph{TelRecord}}  \\
\hline 
{\rm{PRC}$\%$} & 54.05 & 61.39 & 66.54 & 66.47\\
\hline 
{\rm{RCL}$\%$} & 79.97 & 81.72 & 85.47 & 84.83\\
\hline 
$F_1$ & 0.6451 & 0.7011 & 0.7483 & 0.7454\\
\hline 
\end{tabular}
\caption{Segmentation results on \emph{Hub4m97} and \emph{TelRecord}. \label{tab SegResult}}
\end{table}
\section{Conclusion\label{Sec: Con}}
In this paper, we introduced BwDPPs, a class of DPPs where the kernel is almost block diagonal and thus can allow efficient block-wise MAP inference. Moreover, BwDPPs are demonstrated to be useful in change-point detection problem. The BwDPP-based change-point detection method, BwDppCpd, shows superior performance in experiments with several real-world datasets.

The almost block diagonal kernels suit the change-point detection problem well, but BwDPPs may achieve more than that. Theoretically, BwDPP-MAP could be applied to any block tridiagonal matrices without modification. It remains to be studied the theoretical issues regarding exact or approximate partition of a DPP kernel into the form of an almost block diagonal matrix \cite{acer2013recursive}. Other potential BwDPP applications are also worth further exploration.

\section{Appendix: Proof of Lemma \ref{lm apr}}
\begin{proof}
Define
\begin{equation}
\mathbf{S}^i=\left\{\begin{array}{cl}
\mathbf{L} & i=0\\
\begin{bmatrix}
\mathbf{\tilde{L}}_{\mathcal{Y}_{i+1}} & [\mathbf{L}_{\mathcal{Y}_{i+1},\mathcal{Y}_{i+2}} ~ \mathbf{0}]\\ [\mathbf{L}_{\mathcal{Y}_{i+1},\mathcal{Y}_{i+2}} ~ \mathbf{0}]^T & \mathbf{L}_{\cup_{j=i+2}^m \mathcal{Y}_j}
\end{bmatrix} & i=1,\cdots,m-2 \\
\mathbf{\tilde{L}}_{\mathcal{Y}_{i+1}} & i=m-1
\end{array}\right..
\end{equation}
For $i=1,\cdots,m-1$, $\mathbf{S}^i$ is the Schur complement of $\mathbf{\tilde{L}}_{C_i}$ in $\mathbf{S}^{i-1}_{C_{i}\cup(\cup_{j=i+1}^m \mathcal{Y}_j)}$, the sub-matrix of $\mathbf{S}^{i-1}$. We next prove the lemma using the first principle of mathematical induction. State the predicate as: 
\begin{itemize}
\item $P(i)$: $\mathbf{S}^{i-1}$ and $\mathbf{\tilde{L}}_{\mathcal{Y}_i}$ are positive semi-definite (PSD).
\end{itemize}

$P(1)$ trivially holds as $\mathbf{\tilde{L}}_{\mathcal{Y}_1}=\mathbf{L}_1$ and $\mathbf{S}^{0}=\mathbf{L}$ are PSD.

Assuming $P(i)$ holds. $\mathbf{S}^{i-1}_{C_i\cup(\cup_{j=i+1}^m \mathcal{Y}_j)}$ is PSD because $\mathbf{S}^{i-1}$ is PSD. Since $\mathbf{\tilde{L}}_{C_i} \succ 0$ (footnote \ref{Assump}) and $\mathbf{S}^{i}$ is the Schur complement of $\mathbf{\tilde{L}}_{C_i}$ in $\mathbf{S}^{i-1}_{C_{i}\cup(\cup_{j=i+1}^m \mathcal{Y}_j)}$, $\mathbf{S}^{i}$ is PSD. Being sub-matrix of $\mathbf{S}^{i}$, $\mathbf{\tilde{L}}_{\mathcal{Y}_{i+1}}$ is also PSD. Hence, $P(i+1)$ holds. 

Therefore, for $i=1,\cdots,m$, $\mathbf{\tilde{L}}_{\mathcal{Y}_i}$ is PSD.
\end{proof}

\section{Appendix: Proof of Theorem \ref{thrm connection}}
For preparation, first I need to quote a result from \cite{kulesza2012determinantal}: the conditional kernel 
\begin{equation}\label{eq con kernel}
\left(\mathbf{L}\vert A^{in}\subseteq Y, A^{out}\cap Y=\emptyset\right) = \left( \left[ (\mathbf{L}_{\bar{A}^{out}}+\mathbf{I}_{\bar{A}^{in}})^{-1}\right]_{\bar{A}^{in}}\right)^{-1} - \mathbf{I}. 
\end{equation}
Next I need to use the following lemma:
\begin{lemma}\label{lm Lc inv}
$(\mathbf{L}_{\hat{C}_{1:i}}^{-1})_{\hat{C}_i}=\tilde{\mathbf{L}}_{\hat{C}_i}^{-1}$, for $i=1,...,m$, where $\tilde{\mathbf{L}}_{\hat{C}_i}$ is defined by (\ref{def: L tilde C}).
\end{lemma}
\begin{proof}
The proof is given by mathematical induction. 
When $n=1$, the result trivially holds: 
\begin{equation}
(\mathbf{L}_{\hat{C}_{1}}^{-1})_{\hat{C}_1}=\mathbf{L}_{\hat{C}_1}=\tilde{\mathbf{L}}_{\hat{C}_1}^{-1}.
\end{equation}
Assume the result holds for $n=i-1$, i.e.,
\begin{equation}
(\mathbf{L}_{\hat{C}_{1:i-1}}^{-1})_{\hat{C}_{i-1}}=\tilde{\mathbf{L}}_{\hat{C}_{i-1}}^{-1}.
\end{equation}
Consider the case when $n=i$. One has 
\begin{equation}
\begin{split}
&(\mathbf{L}_{\hat{C}_{1:i}}^{-1})_{\hat{C}_i}=(\mathbf{L}_{\hat{C}_i}-\mathbf{L}_{\hat{C}_{1:i-1},\hat{C}_i}^T \mathbf{L}_{\hat{C}_{1:i-1}}^{-1}\mathbf{L}_{\hat{C}_{1:i-1},\hat{C}_i})^{-1}\\
&=(\mathbf{L}_{\hat{C}_i}-\mathbf{L}_{\hat{C}_{i-1},\hat{C}_i}^T (\mathbf{L}_{\hat{C}_{1:i-1}}^{-1})_{\hat{C}_{i-1}}\mathbf{L}_{\hat{C}_{i-1},\hat{C}_i})^{-1}\\
&=(\mathbf{L}_{\hat{C}_i}-\mathbf{L}_{\hat{C}_{i-1},\hat{C}_i}^T \tilde{\mathbf{L}}_{\hat{C}_{i-1}}^{-1}\mathbf{L}_{\hat{C}_{i-1},\hat{C}_i})^{-1}=\tilde{\mathbf{L}}_{\hat{C}_i}^{-1}.
\end{split}
\end{equation}
Therefore the result holds for $i=1,...,m$.
\end{proof}

To prove Theorem \ref{thrm connection}, it suffices to show that 
\begin{equation}
\tilde{\mathbf{L}}_{\mathcal{Y}_i}=\left(\mathbf{L}_{\cup_{j=1}^i \mathcal{Y}_j} \vert \hat{C}_{1:i-1} \subseteq Y, \bar{\hat{C}}_{1:i-1}\cap Y=\emptyset\right). 
\end{equation}
Using (\ref{eq con kernel}) one has 
\begin{equation}
\begin{split}
& \left(\mathbf{L}_{\cup_{j=1}^i \mathcal{Y}_j} \vert \hat{C}_{1:i-1} \subseteq Y, \bar{\hat{C}}_{1:i-1}\cap Y=\emptyset\right)\\
&= \left( \left[ (\mathbf{L}_{\hat{C}_{1:i-1}\cup \mathcal{Y}_i}+\mathbf{I}_{\mathcal{Y}_i})^{-1}\right]_{\mathcal{Y}_i}\right)^{-1} - \mathbf{I}\\
&=\mathbf{L}_{\mathcal{Y}_i}-\mathbf{L}_{\hat{C}_{1:i-1},\mathcal{Y}_i}^T \mathbf{L}_{\hat{C}_{1:i-1}}^{-1}\mathbf{L}_{\hat{C}_{1:i-1},\mathcal{Y}_i}\\
&=\mathbf{L}_{\mathcal{Y}_i}-\mathbf{L}_{\hat{C}_{i-1},\mathcal{Y}_i}^T (\mathbf{L}_{\hat{C}_{1:i-1}}^{-1})_{\hat{C}_{i-1}}\mathbf{L}_{\hat{C}_{i-1},\mathcal{Y}_i}.
\end{split}
\end{equation}
Following Lemma \ref{lm Lc inv} to complete the proof 
\begin{equation}
RHS=\mathbf{L}_{\mathcal{Y}_i}-\mathbf{L}_{\hat{C}_{i-1},\mathcal{Y}_i}^T \tilde{\mathbf{L}}_{\hat{C}_{i-1}}^{-1} \mathbf{L}_{\hat{C}_{i-1},\mathcal{Y}_i}=\tilde{\mathbf{L}}_{\mathcal{Y}_i}.
\end{equation}

\bibliography{ref}
\bibliographystyle{aaai}
\end{document}